\documentclass{article}

\PassOptionsToPackage{numbers, compress}{natbib}
\usepackage[final]{neurips_2023}
\usepackage{xr}
\usepackage[utf8]{inputenc} 
\usepackage[T1]{fontenc}    
\usepackage{hyperref}       
\usepackage{url}            
\usepackage{booktabs}       
\usepackage{amsfonts}       
\usepackage{nicefrac}       
\usepackage{microtype}      
\usepackage{xcolor}         
\usepackage{graphicx}
\usepackage{wrapfig}

\usepackage{soul}
\usepackage{url}
\usepackage{caption}
\usepackage{amsmath}
\usepackage{amsthm}
\usepackage{thmtools}
\usepackage{booktabs}
\usepackage{algorithm}
\usepackage{algorithmic}
\usepackage{amssymb}
\usepackage{wrapfig}
\usepackage{makecell}
\usepackage{graphicx}
\usepackage{textcomp}
\usepackage{xcolor}
\usepackage{multicol}
\usepackage{multirow}
\usepackage{url}
\usepackage{bibentry}
\newtheorem{lemma}{Lemma}
\newtheorem{definition}{Definition}
\newtheorem{proposition}{Proposition}

\newtheorem{corollary}{Corollary}

\usepackage[utf8]{inputenc} 
\usepackage[T1]{fontenc}    
\usepackage{hyperref}       
\usepackage{url}            
\usepackage{booktabs}       
\usepackage{amsfonts}       
\usepackage{nicefrac}       
\usepackage{microtype}      
\usepackage{xcolor}         
\usepackage{subfigure}
\usepackage{changes}
\usepackage{lipsum}

\title{SFP: Spurious Feature-targeted Pruning for Out-of-Distribution Generalization}

\author{%
Yingchun Wang*$^\dag$, Jingcai Guo*$^\dag$, Yi Liu$^\dag$, Song Guo$^\dag$, \\
\textbf{Weizhan Zhang$^\ddag$, Xiangyong Cao$^\ddag$, Qinghua Zheng$^\ddag$} \\
  \textsuperscript{$^\dag$}Department of Computing, The Hong Kong Polytechnic University, Hong Kong SAR\\
  \textsuperscript{$^\ddag$}Department of Computer Science and Technology, Xi’an Jiaotong University, China\\
}

\begin{document}

\maketitle

\begin{abstract}
Model substructure learning aims to find an invariant network substructure that can have better out-of-distribution (OOD) generalization than the original full structure. 
Existing works usually search the invariant substructure using modular risk minimization (MRM) with fully exposed out-domain data, which may bring about two drawbacks: 1) \textit{Unfairness}, due to the dependence of the full exposure of out-domain data; and 2) \textit{Sub-optimal OOD generalization}, due to the equally feature-untargeted pruning on the whole data distribution.
Based on the idea that in-distribution (ID) data with spurious features may have a lower experience risk, in this paper, we propose a novel \textbf{\underline{S}}purious \textbf{\underline{F}}eature-targeted model \textbf{\underline{P}}runing framework, dubbed \textbf{SFP}, to automatically explore invariant substructures without referring to the above drawbacks. 
Specifically, SFP identifies spurious features within ID instances during training using our theoretically verified task loss, upon which, SFP attenuates the corresponding feature projections in model space to achieve the so-called spurious feature-targeted pruning. 
This is typically done by removing network branches with strong dependencies on identified spurious features, thus SFP can push the model learning toward invariant features and pull that out of spurious features and devise optimal OOD generalization. 
Moreover, we also conduct detailed theoretical analysis to provide the rationality guarantee and a proof framework for OOD structures via model sparsity, and for the first time, reveal how a highly biased data distribution affects the model's OOD generalization. 
Experiments on various OOD datasets show that SFP can significantly outperform both structure-based and non-structure-based OOD generalization SOTAs, with accuracy improvement up to 4.72\% and 23.35\%, respectively\footnote{Equal contribution ($*$): Yingchun Wang and Jingcai Guo contributed equally to this work.}.
\end{abstract}
\section{Introduction}
Deep machine learning has demonstrated its excellent capabilities in various fields such as computer vision, natural language processing, recommender systems, etc.~\citep{jordan2015machine}. 
%
However, when faced with real-world fickle data distributions, most applications, are born to be vulnerable due to the ideal assumption that the data are identical and independently distributed. 
In this work, we focus on spurious correlations in out-of-distribution (OOD) settings that are prone to change in real-world data distributions. 
For example, to recognize camels, the learned models usually tend to make use of some spurious correlations, i.e., camels are commonly associated with a yellow desert background. However, the generalization can be poor if the camels are photographed in green oases or even less background \citep{nui}.
Therefore, it is crucial to improve the OOD generalization ability of these models. 
In recent years, a wealth of literature has been generated in this field. 
%
For example, \citep{survey} summarizes several popular branches under the supervised setting, including domain generalization, causal invariant learning, and stable learning. 
Specifically, domain generalization (DG)~\citep{dg1,dg2} combines multiple source domains to learn models that generalize well on unseen target domains. 
Differently, causal learning and invariant learning~\citep{cl1,iv1,iv2} explore the invariance of data predictions in a more principled way for causal inference. 
In another way, stable learning \citep{st1,st2} aims to establish a consensus between causal inference and machine learning to improve the robustness and credibility of models.

Most recently, some works address the OOD problem from the perspective of model structure, which is also our focus. Compared with the above-mentioned methods, the model-structure-based approach has the extra advantage that it is general and can be embedded in most SOTAs to further improve their performance. 
For example, \citep{ar1} provides sufficient and intuitive motivation for this branch of OOD generalization, and suggests that over-parameterized models could degrade OOD performance through data memorization and overfitting. Differently, \citep{can} claims that even highly spurious feature-related full networks can contain particular substructures that may achieve better OOD generalization, and proposes a module detection technique, with the guidance of fully exposed out-domain data, to identify this functional lottery. 

However, despite the progress made, the model-structure-based methods are mostly empirically constructed and lack theoretical explanations and proof of effectiveness. 
One may note that previous methods usually apply existing non-OOD-specific techniques such as network architecture search and module detection to find OOD lottery tickets, which may degrade the effectiveness of these techniques in OOD setting. 
For example, \citep{can} indicates that the sparsity of the weights is not exactly the sparsity of the model about spurious features in their method. 
Worse still, most approaches rely on the guidance of fully exposed OOD data to find the target substructure, which is highly unlikely to be feasible in real-world applications.

To address the above issues, we propose a novel spurious feature-targeted model pruning framework, dubbed SFP, to automatically explore the optimal invariant model substructure with better OOD generalization. 
Specifically, SFP can identify spurious features (correlations) within ID instances with high probability during training without the full exposure of out-domain data, thus preventing the model from fitting such identified features and executing model sparsity, particularly for spurious features. 
%
To do it, we build two subspaces spanning from highly biased training data to provide the coordinate basis for spurious and invariant features, respectively, and further build a model space as the reference for feature projection, i.e., from learned features to the model.
In practice, our idea is partially motivated by the finding that the nonlinear activations of CNNs can change implicitly into a ``coupled'' manner with linear architectures (i.e., linear classifiers)~\citep{du2018}, and such balanced invariant has been verified feasible under OOD scenarios~\citep{finetune}. 
Based on that, we prove that input data instances with smaller prediction losses can contain more significant spurious features during training, which is then used as the rationale for spurious feature identification. 
By weakening the feature projections only for those identified features into the model space, we can increase the resistance of the model space basis to learn towards directions of the subspace spanning from spurious features. 
As a result, we can progressively adjust the component rank in the projection matrix to the ordering of invariant feature correlation via singular value decomposition (SVD), i.e., the directions of the model space corresponding to the lowest singular values are sparsed out. 
%
In summary, our contributions are three-fold:
\begin{itemize}
    \item We provide a framework for proving the rationality and efficacy of pruning better OOD substructures through model sparsity, thereby compensating for the absence of theoretical guidance in previous work in this field.
    \item We propose a novel spurious feature-targeted model pruning method to automatically find OOD substructures during training, totally without prior causal assumptions nor the full exposure of additional out-domain data.
    \item To the best of our knowledge, we are the first to theoretically reveal the correspondence between the biased data features and the model substructures for better OOD generalization.
\end{itemize}

\section{Related Work}
\textbf{Out-of-Distribution Generalization.}
\label{ood related work}
In recent years, a number of efforts have been made to address the distribution shift between train-test data, namely, OOD generalization. 
Existing literature can be roughly divided into two categories including non-structure-based methods and structure-based methods. 
Specifically, the non-structure-based methods focus on the feature level and usually limit models over learning on spurious features by designing heuristic learning paradigms or separating different features in high dimensions. 
For example, \citep{IRM} aims to extract nonlinear invariant predictive features across multiple environments. 
IIB \citep{iib} performs invariant feature prediction by limiting the mutual information between the learned representation and the ground truth. 
In another way, there are also some works focus on feature disentanglement, which separates the representations of different variables in data \citep{DBLP:journals/pami/BengioCV13,DBLP:conf/iclr/LocatelloBLRGSB19,DBLP:journals/pieee/ScholkopfLBKKGB21}. 
However, existing non-structure-based methods only focus on the training process of data while ignoring the influence of model structures.

Differently, the structure-based methods aim to investigate the impact of model structures on OOD generalization. Early work can be traced back to \citep{peters2016causal}, which affirms that models with specific structures under linear conditions can avoid false correlations in OOD generalization. 
Although not limited to OOD problems, lottery theory \citep{frankle2018lottery} claims a viewpoint similar to that of \citep{peters2016causal} under nonlinear conditions. 
Most recently, \citep{can} proposes the functional lottery hypothesis, which further confirms the improvement of model structure on OOD generalization performance under OOD setting and nonlinear condition. 
Moreover, this positive impact can be superimposed on most previous non-structure-based methods. 
However, these methods only utilize model compression algorithms while ignoring the relationship between data features and model structures. 

\noindent\textbf{Model Pruning.}
\label{sec: pruning related work}
A series of network pruning methods have been proposed to eliminate unnecessary weights from over-parameterized networks. 
Early research \citep{DBLP:conf/nips/CunDS89} usually tries to remove weight parameters based on the Hessian matrix of the objective function. similarly, \citep{DBLP:conf/nips/HanPTD15} proposes to remove the weights or nodes with small-norm from DNNs. However, these kinds of unstructured pruning (i.e., discrete weights or nodes) can hardly reduce reasoning time without specialized hardware \citep{DBLP:conf/nips/WenWWCL16}. 
Therefore, structured pruning~\citep{DBLP:conf/nips/WenWWCL16,DBLP:conf/iclr/0022KDSG17}, i.e., channels/filters, is more applicable and becomes mainstream. 
For example, \citep{DBLP:conf/ijcai/HeKDFY18} resets less important filters at every epoch while updating all other filters. \citep{DBLP:conf/cvpr/ZhaoNZZZT19} uses stochastic variational inference to remove the channels with smaller mean/variance. 
Despite all that, previous methods essentially follow the traditional empirical risk-guided model pruning paradigm, thus the obtained feature-untargeted sparse model is suboptimal for OOD generalization.

\section{Proposed Method}

We start by formalizing the model structure-based OOD problem in a complete \textit{inner product space} and then provide a theoretical analysis to investigate the impact of ID data and out-domain data on model performance. 
Based on this framework, we elaborate on the optimization objective of SFP and theoretically demonstrate its effectiveness. 
\subsection{Notations and Preliminaries}\label{sec:3.1}
\subsubsection{Linear Parameterized Notations}
Define $X_{id}\in \mathbb{R}^{p \times d}$ and $X_{ood} \in \mathbb{R}^{q \times d}$ as two datasets to present ID data and out-domain data, respectively, where $p$, $q$ denote data numbers, and $d$ is the feature dimension. Thus, the whole training dataset is denoted as $X=X_{id}\cup X_{ood}$, where $X \in \mathbb{R}^{n \times d}$ and $n=p+q$, and $Y$ represents the corresponding ground truth of the feature projections. Note that $p \gg q$ is defined for the problem setting of a large OOD shift. 
Given $p_i$ and $p_o$ as the proportion of ID instances and out-domain instances in the training dataset, we limit $p_i \gg p_o$ and $p_i + p_o = 1$. 
Next, we use $\mathcal{W}\in \mathbb{R}^{m\times d}$ as the parameters of the feature extractor in CNNs, where $m$ is the dimension of the feature map output before the classification layer. 
To rebuild the problem of OOD generalization in a complete inner product space, some auxiliary notations are defined as follows. Let $R = \boldsymbol{C}(\mathcal{W}^\top)$, $S = \boldsymbol{C}(X_{id}^\top)$, and $U = \boldsymbol{C}(X_{ood}^\top)$ 
be the subspace spanning the parameterized model, ID data, and out-domain data by their rows, respectively. 
Let $E \in \mathbb{R}^{d \times \text{dim}(R)}$, $F \in \mathbb{R}^{d \times \text{dim}(S)}$, and $G \in \mathbb{R}^{d \times \text{dim}(U)}$ as the basis of the orthogonal matrix with standard columns and rows $R$, $S$, and $U$, respectively. 
Then, the original algebraic representation of the model and dataset can be reformulated in linear form as spanning spaces over a set of learnable basis vectors. 
Based on ``the deep multi-layer homogeneity'' suggested in \citep{du2018}, we approximate the training trajectory under OOD settings in a linear form. Thus, 
we propose the proposition and analyze the details of the linear transformation as follows:

\begin{proposition}
Model substructures and the feature representations can be effectively corresponded in linear form by the singular value decomposition (SVD) of the feature projections of data into the model space.
\label{proposition1}
\end{proposition}

\noindent \textbf{Discussion (Model):}
Define $E^\top F \in \mathbb{R}^{\text{dim}(R) \times \text{dim}(S)}$ as the basis of $\boldsymbol{C}(X_{id} \mathcal{W}^\top)$ spanning the ID (spurious) feature projections. 
Similarly, $E^\top G \in \mathbb{R}^{\text{dim}(R) \times \text{dim}(U)}$ is the basis of $\boldsymbol{C}(X_{ood} \mathcal{W}^\top)$ spanning the out-domain feature projections. 
Since the column of $E$ span $R$, we have $\mathcal{W} = Er$ for some $r \in \mathbb{
R}^{\text{dim}(R)}$. 
For every ID instance, the feature projection $r_1=E^\top Fa$ is used for some $a \in \boldsymbol{C}(E^\top F)$, where $a$ is a column vector of $\mathbb{R}^{\text{dim}(S)}$. 
Similarly, for every out-domain instance, the feature projection $r_2=E^\top Gb$ is used for some $b \in \boldsymbol{C}(E^\top G)$, where $b$ is a column vector of $\mathbb{R}^{\text{dim}(U)}$. Therefore, the feature projections of the whole training dataset in the model space can be defined as $r=p_ir_1+p_or_2$. 
Assume $\mathcal{W}^*$ is the optimal set of model parameters, $\mathcal{W^*}=Er^*$, where $r^*= p_iE^\top Fa^* + p_o E^\top G b*$, and $a^*$, $b^*$ be the true feature projections. 

\noindent \textbf{Discussion (Data):} 
In $S=\boldsymbol{C}(X_{id}^\top)$ with basis $F$ spanning $X_{id}$, 
$\forall~x_i \in X_{id}$, $\exists~z \in \mathbb{R}^{\text{dim}(S)}$, $x = (Fz)^\top$. $X_{id}=(FZ)^\top$, where $Z=\{z\}$. 
Similarly, in $U=\boldsymbol{C}(X_{ood}^\top)$ with basis $G$, $\forall~x_{ood} \in X_{ood}$, $\exists~v \in \mathbb{R}^{\text{dim}(U)}$, $x_{ood} = (Gv)^\top $. $X_{ood}=(GV)^\top$, where $V=\{v\}$. 

\subsubsection{Preliminary Optimization Target.}
\begin{definition}
Under OOD setting, applying the same optimization objective to ID data with spurious features and out-domain data without the same spurious features is called undirected learning
\end{definition}
\begin{definition}
Trained independently from scratch for the same number of iterations, the substructure within the original model having the best OOD generalization performance is defined as the OOD lottery~\citep{can}.
\end{definition}

\noindent For the structure-based approach searching the OOD lottery based on undirected learning, the optimization target can be formulated as:
\begin{equation}
    \min~ \mathcal{L}(\mathcal{W}, X, Y) = \mathbb{E}_{X}\left \| X \mathcal{W} -Y \right \|_2^2 + \mathcal{S}\left( \mathcal{W}\right),
    \label{1}
\end{equation}
where $\mathcal{L}$ is the task-dependent loss function, and $\mathcal{S}$ is the function that induces the sparsity of the model structure to find the target subnetwork. 
The domain-generalized substructure is described by layer-wise channel saliencies in SFP. To this end, $\mathcal{S}$ is implemented by the squeeze-and-excitation module as suggested in \citep{se}. 
%
%
The value of relevant parameters in $t$-th iteration is represented by subscript $t$, and the optimal value is represented by superscript $*$. Thus, the task loss in $t$-th iteration can be calculated as:
\begin{equation}
\setlength{\abovedisplayskip}{1pt}
\setlength{\belowdisplayskip}{1pt}
 \begin{aligned}
    \mathcal{L}_t&=\left \| X \mathcal{W}_t -Y \right \|_2^2=\left \| X \mathcal{W}_t -X \mathcal{W}^* \right \|_2^2,
\end{aligned}
\label{2}
\end{equation}
and the gradient is:
\begin{equation}
    \frac{\partial \mathcal{L}_t}{\partial \mathcal{W}_t} 
    = 2\left(\mathcal{W}_t -\mathcal{W}^* \right) X^\top X.
\label{3}
\end{equation}
The orthogonal basis of the model space is regarded as the left singular vectors when performing SVD on the feature projections of data. The right singular vectors correspond to input data features, and the corresponding singular values can be defined as indicators of the importance of the current data features w.r.t. the model structure.
To internally observe the impact of ID and out-domain features on the model, the gradient accumulation is further transformed into a linear form as: 
\begin{equation}
\setlength{\abovedisplayskip}{3pt}
\begin{aligned}
    \frac{\partial \mathcal{L}_t}{\partial \mathcal{W}_t} 
    %
    %
    = 2(p_i^2 (a_t-a^*)\Sigma_{E^\top F,t}^2 X_{id} 
    + p_o^2 (b_t-b^*)\Sigma_{E^\top G,t}^2X_{ood}),
\end{aligned}
\label{4}
\end{equation}
where $\Sigma$ denotes the corresponding singular value matrix, and for simplicity, we omit $t$ under $\Sigma$ in the following discussion. 
The proof of Eq.~\ref{4} is provided in \underline{\textit{Appendix A.1}}.

Since $\text{dim}(U)=q \ll \text{dim}(S)=p$, we have $\min~\Sigma_{F^\top G} = \min~\Sigma_{G^\top F} = \sigma_{G^\top F}^q$. Similarly, $\min~\Sigma_{F^\top E} = \min~\Sigma_{E^\top F} = \sigma_{E^\top F}^m$, and $\min~\Sigma_{G^\top E} = \min~\Sigma_{E^\top G} = \sigma_{E^\top G}^m$. 
Finally, the model parameters can be calculated as: 
\begin{equation}   
\setlength{\abovedisplayskip}{1pt}
\setlength{\belowdisplayskip}{1pt}
\begin{aligned}
    \min \mathcal{W}^\infty  
    %
    %
    =&\mathcal{W}_0 - 2lr\sum_{t=1}^{\infty }\sum_{i=1}^{m} p_i^2(a_t-a^*)\sigma_{E^\top F,t, i}^2 X_{id} \\
    &- p_o^2 (b_t-b^*) \sigma_{E^\top G,t,i}^2 X_{ood}.
\end{aligned}  
\label{5}
\end{equation}
\subsubsection{Biased Performance on Out-domain and ID Data}
Based on the gradient flow trajectories, we compare the learning process and final performance of the model for spurious and invariant features, respectively. 
We observe that the model structure obtained by undirected learning has an obvious performance difference between ID data and out-domain data. 
With this observation, we propose the following propositions. 
\begin{proposition}
Undirected learning (full or sparse training) on biased data distributions can lead to significantly different forward speeds of the model learning along different data feature directions, and the difference has a second-order relationship with the proportion of different data distributions in the training set, i.e.:
\begin{equation}   \setlength{\abovedisplayskip}{1pt}
\setlength{\belowdisplayskip}{1pt}
\begin{aligned}
   \left|\frac{\partial W_t}{\partial (a_t-a^*)} - \frac{\partial W_t}{\partial (b_t-b^*)}\right| \approx 2(p_i^2 \Sigma_{E^\top F}^2 - p_o^2 \Sigma_{E^\top G}^2 ).
\end{aligned}  
\label{6}
\end{equation}
\label{proposition2}
\end{proposition}
\noindent\textbf{Discussion (Update Gradient):} We compute the direction gradients along the directions of the feature projections of ID and out-domain data, respectively. As shown in Eq.~\ref{6}, with $p_i>p_o$, the learning of the basis of the model space is gradually biased towards the directions of spurious features. By performing SVD on the projection of the basis vector of the feature space through the model space, the obtained singular value matrix can be regarded as the fitting degree of the model on the corresponding data distribution at $t_{th}$ iteration.

\begin{proposition}
Undirected learning (full or sparse training) on biased data distributions causes the model to be more biased towards training features with a larger proportion, bringing about significant performance differences in different data distributions, i.e.:
%
%
\begin{equation}   
\begin{aligned}
\mathcal{L}_{ood} &-\mathcal{L}_{id}
\approx(p_i^2-p_{o}^2)(1-\Sigma_{F^\top G}) + \epsilon >0,
\end{aligned} 
\label{7}
\end{equation}
where $\epsilon$ is the difference of initial feature projections between ID and out-domain data due to model initialization error. The full proof of Eq.~\ref{7} is provided in \underline{\textit{Appendix A.2}}.
\label{proposition3}
\end{proposition}

Taking the risk difference between ID data and out-domain data of the trained model as the measurement of the OOD generalization, the following conclusion is derived:
\begin{corollary}
Undirected learning of networks on highly biased training domains (the dataset consists of a majority data group with spurious features) can only lead to substructures with sub-optimal OOD generalization performance.
\end{corollary}
\noindent\textbf{Discussion (Performance Difference):} 
The result intuitively shows that the undirectly learned model performs better on feature distributions with larger instance numbers. 
As shown in Eq.~\ref{7}, the difference in model performance between out-domain data and ID data is linearly related to the proportion of the corresponding instances and the correlation degree between the different feature distributions. 
Moreover, when the out-domain data has the same proportion as ID data in the training dataset (i.e., $p_i=p_o$) or the data distributions of them are consistent, the task loss difference between out-domain and ID data can be reduced to zero. 


\subsection{SFP: An Spurious Feature-Targeted Model Pruning Method}{\label{sec:3.2}}
To address the problem of sub-optimal OOD substructure caused by undirected training, 
we propose a novel method to effectively remove model branches that are only strongly correlated with spurious features. 
As demonstrated in Figure.~\ref{f1}, the pipeline consists of two stages, including spurious feature identification and model sparse training. 
Specifically, SFP identifies large spurious feature components within ID instances with high probability by observing the loss during training. It then can perform spurious feature-targeted model sparsity by analyzing the SVD of the feature projection matrix between the data and model space. 
We also provide a detailed theoretical analysis of both stages of the proposed SFP in the following part. 
\begin{figure*}
\centering
\setlength{\abovecaptionskip}{2pt}
\includegraphics[width=0.87\textwidth]{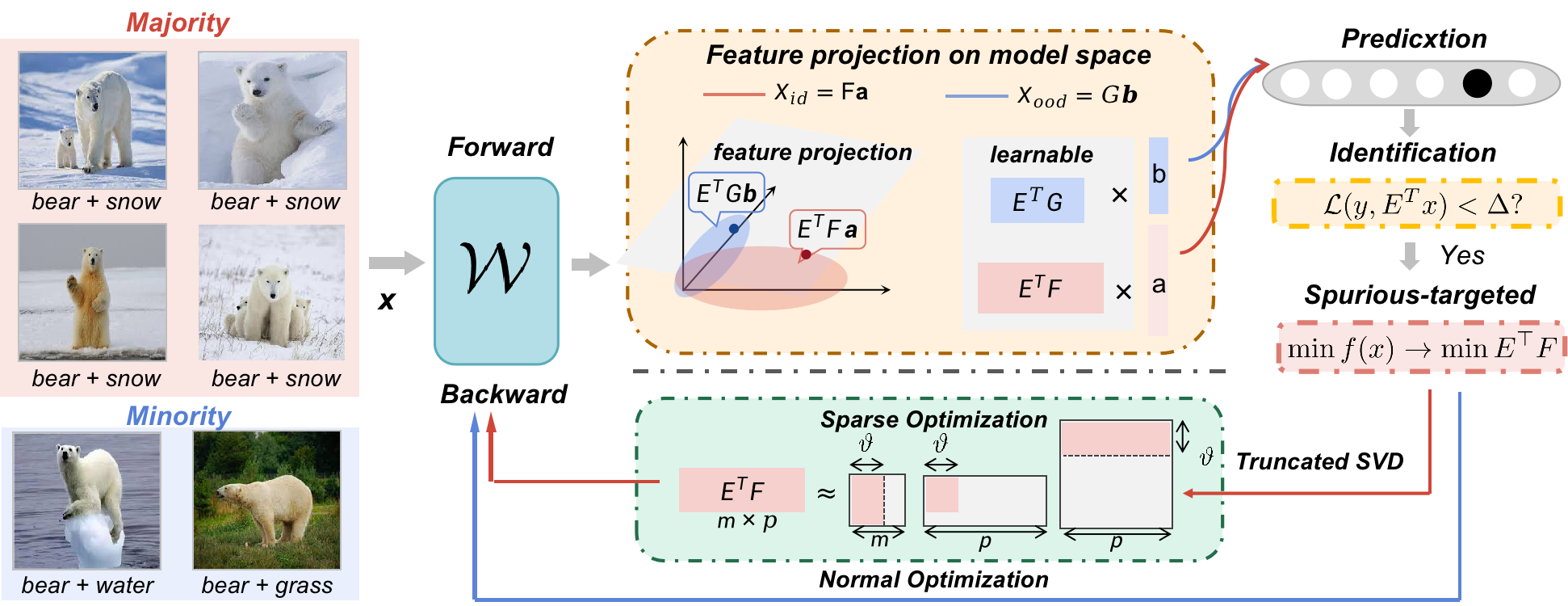}
\caption{The training pipeline of SFP.}
\label{f1}
\end{figure*}

\subsubsection{Spurious Feature Identification}

As shown in Proposition.\ref{proposition3}, 
if no intervention is applied, a model trained on a highly biased data distribution can be gradually biased towards ID data with lower prediction loss. Since the loss difference between ID and out-domain data can be approximately computed by $(p_i^2-p_o^2)(1-\sigma_{F^\top G})$, it is, therefore, can be adopted as the identification criterion for spurious features in each iteration. 
In brief, if the loss corresponding to the current data is lower than a threshold $\Delta$, then the current data is likely to be an ID instance dominated by spurious features. Then we can further prune the spanning sets of model space along the directions of these spurious feature projections. 
%
%
To compute $\Delta$, we first investigate the average loss in the $t-1$-th iteration as:
\begin{equation}
    \begin{aligned}
    \bar{\mathcal{L}}^{t-1}\approx \mathcal{L}^{t-1}_{id} + p_o(p_i-p_o)(1-\sigma_{F\top G}^{t-1}).
    \end{aligned}
\label{8}
\end{equation}

\begin{wrapfigure}{r}{6cm}
\centering
\includegraphics[width=0.42\textwidth]{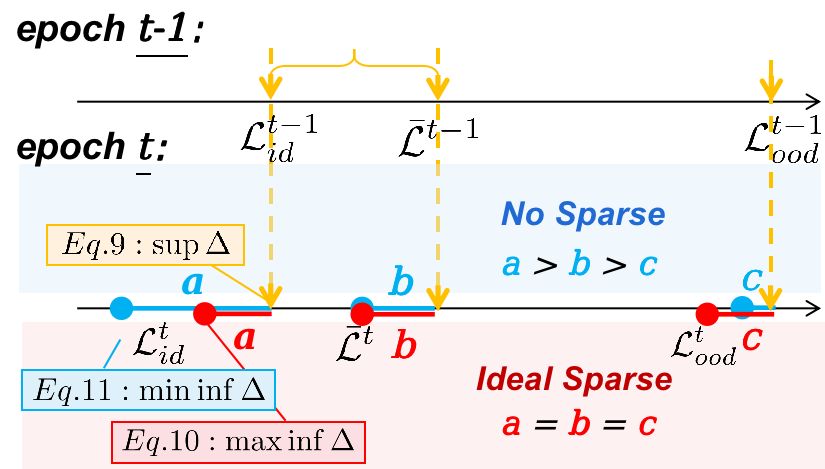}
\caption{Identification of the ID instances dominated by spurious features. At epoch t, if no intervention is applied, the average loss drop on all data (blue b) should be smaller than the loss drop on ID data (blue a) and larger than the drop on out-domain data (blue b). The red line denotes an ideal regularization effect: the loss drops uniformly on all data.}
\label{f2}
\end{wrapfigure}

As shown in Figure. ~\ref{f2}, since $\mathcal{L}_{id}^t<\mathcal{L}_{id}^{t-1}$, we have: 
\begin{equation}
    \begin{aligned}
    \sup \mathcal{L}^{t}_{id}=|\bar{\mathcal{L}}^{t-1}-p_o(p_i-p_o)(1-\sigma_{F\top G}^{t-1})|.
    \end{aligned}
\label{9}
\end{equation}
Similar with Eq.~\ref{8}, the lower bound of the loss on ID data at $t$-th iteration can be computed as:
\begin{equation}
    \begin{aligned}
    \inf \mathcal{L}^{t}_{id}
    = |\bar{\mathcal{L}}^{t}-p_o(p_i-p_o)(1-\sigma_{F\top G}^{t})|.
    \end{aligned}
\label{10}
\end{equation}
The spurious feature-targeted regularization forces the model to learn invariant features and achieve fair loss reduction on all instances: $|\mathcal{L}_{id}^t-\mathcal{L}_{id}^{t-1}|=|\bar{\mathcal{L}}^t-\bar{\mathcal{L}}^{t-1}|$. Therefore, the ideal lower bound of the ID loss at $t$-th iteration is:
\begin{equation}
    \begin{aligned}
    \inf \mathcal{L}^{t}_{id}
    = &|\bar{\mathcal{L}}^{t-1}-p_o(p_i-p_o)(1-\sigma_{F\top G}^{t-1})|  \\
    &-|\bar{\mathcal{L}^t} - \bar{\mathcal{L}^{t-1}}|.
    \end{aligned}
\label{11}
\end{equation}
Thus, $\mathcal{L}_{id}^t$ is highly likely to be located in the range of $\left[\min \inf \mathcal{L}_{id}^{t},\sup \mathcal{L}_{id}^{t}\right]$. The upper bound is used to compute $\Delta$ for identifying instances dominated by spurious features.  

\subsubsection{Spurious Feature-Targeted Pruning}
SFP reacts to spurious feature-related instances by weakening their corresponding spurious feature projections into the model space, which can prevent the model from over-fitting on identified spurious features. 
To analyze the projections from data into the model space, we define $\Xi \in \mathbb{R}^{m \times m}$, $\Lambda \in \mathbb{R}^{p \times p}$, and $\Gamma \in \mathbb{R}^{q\times q}$ as the normalized orthogonal basis of $\boldsymbol{C}(E^\top E)$, $\boldsymbol{C}(E^\top F)$, and $\boldsymbol{C}(E^\top G)$, spanning the optimal model projections, the feature projections of ID data into the model space, and the feature projections of out-domain data into the model space, respectively. 
$\xi_i$, $\lambda_i$, and $\gamma_i$ denote the $i$-th column vectors in $\Xi$, $\Lambda$, and $\Gamma$, respectively. 
\begin{lemma}
Spurious feature-targeted model sparsity can effectively reduce the performance deviation of the learned model between ID data and out-domain data: 
\begin{equation}
\begin{aligned}
\left|\frac{\mathcal{R}(X_{ood})-\mathcal{R}(X_{id})}{\mathcal{R}(X_{ood})^{sparse}-\mathcal{R}(X_{id})^{sparse}}\right| \approx \left|\frac{\sum_{j=1}^{m}  p_{o} \tilde{\sigma_j} \xi_j \gamma_j X_{ood}- \sum_{i=1}^{m} p_{i} \sigma_i \xi_i \lambda_i X_{id}}{\sum_{j=1}^{m}  p_{o} \tilde{\sigma_j} \xi_j \gamma_j X_{ood}-\sum_{i=1}^{\vartheta}p_{i} \sigma_i \xi_i \lambda_i X_{id}}\right| \ge 1,
\end{aligned}
\label{12}
\end{equation}
where $\mathcal{R}(\cdot)$ is the empirical risk function. $\sigma_{i}$ and $\tilde{\sigma}_i$ is the $i$-th maximum in $\Sigma_{E^\top F}$ and $\Sigma_{E^\top G}$, and we have $\sigma>0$ since the singular values are non-negative. $m$ and $\vartheta$ is the rank of the singular value matrix after performing compact SVD and truncated SVD on the projections, respectively.  
\label{lemma1}
\end{lemma}

\begin{proof}[Proof of \textbf{Lemma.\ref{lemma1}}]
As mentioned earlier, the projection space before the model sparsity can be represented as:
    \begin{equation}
\setlength{\abovedisplayskip}{1pt}
\setlength{\belowdisplayskip}{1pt}
        \begin{aligned}
        Er 
        = \sum_{i=1}^{m}\left ( p_i \sigma_i \xi_i \lambda_i ^\top + p_o \tilde{\sigma_i} \xi_i \gamma_i ^\top \right ).
        \end{aligned}
    \label{13}
    \end{equation}
Specifically, SFP first performs SVD on the feature projections which maps input data to a set of coordinates based on the orthogonal basis of model space. 
The matrices of left and right singular vectors correspond to the standard orthogonal basis of the model space and data space, respectively. The matrix of singular values corresponds to the direction weight of the action vectors in the projection matrix. 
SFP prunes the model by trimming the smallest singular values in $\Sigma$ as well as their corresponding left and right singular vectors. 
In this way, SFP can remove the spurious features in ID data space and substructures in the model space simultaneously in a spurious feature-targeted manner along the directions with weaker actions for projection. 
Then, the projection space with only the most important $\vartheta$ singular values can be formalized as:
    \begin{equation}
        \begin{aligned}
        Er^{sparse} = p_i \Xi \Sigma_{E^\top F} \Lambda^{-1} + p_o \xi \Sigma_{E^\top G} \Gamma^{-1} = \sum_{i=1}^{\vartheta} p_i \sigma_i \xi_i \lambda_i ^\top + \sum_{j=1}^{m} p_o \tilde{\sigma_j} \xi_j \gamma_j ^\top. 
        \end{aligned}
    \label{14}
    \end{equation}
Based on the representation of the projection spaces, the model response to data features $\mathcal{R}(X)=ErX$ can be calculated as:
\begin{equation}
        \begin{aligned}
        \mathcal{R}(X) = \left\{ p_i \Xi \Sigma_{E^\top F} \Lambda^{-1} + p_o \xi \Sigma_{E^\top G} \Gamma^{-1}\right\}^\top X^\top = \sum_{i=1}^{m} \left\{p_i \sigma_i \xi_i \lambda_i ^\top X^\top + p_o \tilde{\sigma_i} \xi_i \gamma_i ^\top X^\top\right\}. 
        \end{aligned}
    \label{15}
    \end{equation}
\end{proof}

\subsection{Correspondence between Model Substructure and Spurious Features}{\label{sec:3.3}}
In this section, we theoretically demonstrate that, with a reasonable setting of the sparse penalty for ID data, SFP can effectively reduce the overfitting of the model on spurious features while retaining the learning on invariant features. 
Specifically, for a training instance $x$ that is identified by SFP as spurious feature-dominant ID data, we define $f(x)$ as the last feature maps output by the model and also the projection of $x$ into the model space to be learned defined on the spanning set $E$. 
We use $x\sim F$ to represent $x\in ID$ since $F$ is the basis of the row space spanning ID training instances, and for this reason, we use $x\sim G$ to represent $x\in OOD$. We simply use $E^\top F$ and $E^\top G$ to denote the projections of the input features in the model space.
Thus, the optimization target of SFP can be formulated as: 
\begin{equation}   
\begin{aligned}
\min_{E}\mathcal{L}=\mathcal{L}_{ce}+ \eta \mathbb{E}_{x\sim F}f(x), 
\end{aligned} 
\label{16}
\end{equation}
where $\eta$ is the sparsity factor imposed on the feature projections for the identified ID data, by which, SFP can adjust model structures via adaptively recalibrating the channel-wise feature responses of the ID data at a rate $\eta$.

\begin{lemma}
Define $e= |f^*(x)-f(x)|$ as the difference between true feature maps $f^*(x)$ and $f(x)$. When $\eta < 2e$, SFP can effectively reduce the learning of the model towards spurious features while keeping the performance on the same features.
\label{lemma2}
\end{lemma}
\begin{proof}[Proof of \textbf{Lemma.\ref{lemma2}:}]
The prediction errors of feature projections $L_{f}$ can be defined as:
\begin{equation}
    \begin{aligned}
L_{f} = |f^*(x)-f_(x)|^2 =\sum_{i,j=j_1\cup j_2}(f^*(x)
-\sigma_{i,j_1}\xi_{i}\top\lambda_{j_1} -\sigma_{i,j_2}\xi_{i}\top\gamma_{j_2})^2,
\end{aligned}
\label{17}
\end{equation}
and the corresponding gradient is:
\begin{equation}
    \begin{aligned}
\frac{\partial L_{f}}{\partial \sigma_{i,j_1}\xi_{i} } =\frac{\partial e^2}{\partial \sigma_{i,j}\xi_{i}}=2e\frac{\partial e}{\partial \sigma_{i,j}\xi_{i}} =2e \frac{\left|f^*(x)
-\sigma_{i,j_1}\xi_{i}\top\lambda_{j_1} -\sigma_{i,j_2}\xi_{i}\top\gamma_{j_2})\right |}{\partial \sigma_{i,j}\xi_{i}} =-2e\lambda_{j_1},
\end{aligned}
\label{18}
\end{equation}
where $i$ and $j$ is the index of column vectors in the orthogonal basis for model space and feature space, respectively.
For out-domain data, the gradient of the column vectors in the OOD projection matrix interacting with the $j_{th}$ feature vector is $-2e\gamma_{j_2}$. 
Then, splitting the in-domain features into the spurious features $F'$ and the invariant features $IN$, and splitting the out-domain features into the unknown features $G'$ and the invariant features $IN$. 
With a high probability under the OOD setting, we assume $F'$ and $G'$ are orthogonal. 
%
To achieve the spurious feature-targeted unlearning and invariant feature-targeted learning of the model, we need to satisfy the following constraint:
\begin{equation}
\begin{aligned}
    2ep_i\lambda_{IN}+2ep_o\gamma_{IN}-p_i\eta \lambda_{IN}>2ep_o\gamma_{G'} 
    \Rightarrow~\eta \le \frac{2ep_i\lambda_{IN}+2ep_o\gamma_{IN}-2ep_o\gamma_{G'}}{p_i \lambda_{IN}}\approx 2e.
\end{aligned}
\label{19}
\end{equation}
\end{proof}
Since the de-learning rate of the spurious feature is positively correlated with $\eta$, the upper bound $\eta=2e$ is taken.
\bigskip
\setlength{\belowdisplayskip}{3pt}

\section{Experiments}{\label{sec:4}}
\subsection{Experimental Setting}
We evaluate the proposed SFP on three constructed OOD datasets, including Full-colored-mnist, Colored-object, and Scene-object. 
As shown in Fig.~\ref{fig: ood dataset}, the invariant features are the focused digits or objects in the foreground, and the spurious features are the background scene. 
In particular, in the Full-colored-mnist and Colored-object datasets, we generate ten pure-colored backgrounds with different colors as the spurious features, and in the Scene-object dataset, we extract ten real-world scenes from PLACE365 dataset~\citep{zhou2017places} as the backgrounds. 
Besides, the objects in Colored-object and Scene-object datasets are extracted from MSCOCO dataset~\citep{lin2014microsoft}.
For all datasets, we design the biased data instances with a one-to-one object-scenery relationship, e.g., in the Full-colored-mnist, the biased instance of digit $1$ always has a pure red background, and the unbiased instance has a background with a randomly assigned color. 
The former is considered an ID instance since it contains spurious features, while the latter is considered an out-domain instance.

\begin{figure}[htbp]
\centering
\subfigure[Full-colored-mnist]{
\begin{minipage}[t]{0.3\linewidth}
\centering
\includegraphics[width=1.1in]{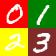}
\end{minipage}%
}%
\subfigure[Colored-object]{
\begin{minipage}[t]{0.3\linewidth}
\centering
\includegraphics[width=1.1in]{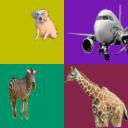}
\end{minipage}%
}%
\subfigure[Scene-object]{
\begin{minipage}[t]{0.3\linewidth}
\centering
\includegraphics[width=1.1in]{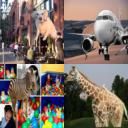}
\end{minipage}
}%
\centering
\caption{Visualization of three constructed OOD datasets.}
\label{fig: ood dataset}
\end{figure}

\begin{table}[!t]
\centering
\fontsize{9}{4.7}\selectfont
\setlength{\tabcolsep}{4.55mm}{
\caption{OOD generalization performance on Full-colored-mnist, Colored-object, and scene-object. `Tr Acc.' / `Te Acc.' denotes the train accuracy and test accuracy, respectively, and `SFP+' denotes the integration with other algorithms.} 
\begin{tabular}{lcccccc}
\toprule
\multirow{2}{*}{{\begin{tabular}[c]{@{}c@{}}Method\end{tabular}}} & \multicolumn{2}{c}{{Full-colored-mnist}}                             & \multicolumn{2}{c}{{Colored-object}} & \multicolumn{2}{c}{{Scene-object}}                                              \\ 
\cmidrule(lr){2-3}\cmidrule(lr){4-5}\cmidrule(lr){2-3}\cmidrule(lr){6-7}
 & Tr Acc.  & Te Acc.   & Tr Acc.  & Te Acc. & Tr Acc. & Te Acc. \\ \midrule \midrule
ERM                     & $93.96$        & $62.20$    & $99.99$        & $59.19$      & $100$         & $27.4$       \\
MRM                     & $96.71$        & $80.95$    & $99.99$        & $60.65$        & $100$         & $26.74$ \\
\underline{\textbf{SFP}}                     & $97.48$        & $84.29$     & $100$          & $61.01$      & $100$         & $28.41$   \\ 
\midrule
IRM                     & $96.25$        & $77.96$     & $99.99$        & $62.88$     & $99.79$       & $36.88$     \\
MODIRM                  & $98.11$        & $89.32$     & $99.98$        & $64.52$     & $99.78$       & $36.92$      \\  
\underline{\textbf{SFP+}}IRM                  & $98.35$        & $89.93$     & $100$          & $65.8$       & $99.76$       & $38.1$ \\ 
\midrule
REX                     & $97.44$        & $87.80$   & $100$          & $64.72$        & $99.76$       & $36.71$    \\
MODREX                  & $98.39$        & $92.19$      & $100$          & $64.52$      & $99.82$       & $36.66$     \\
\underline{\textbf{SFP+}}REX                  & $98.61$        & $93.42$        & $100$          & $66.08$     & $99.68$       & $37.91$   \\ 
\midrule
DRO                     & $94.05$        & $62.89$    & $100$          & $66.76$         & $99.98$       & $31.31$    \\
MODDRO                  & $96.73$        & $80.52$    & $100$          & $66.18$     & $99.85$       & $29.38$  \\
\underline{\textbf{SFP+}}DRO                  & $97.56$        & $85.24$    & $100$          & $68.44$      & $99.98$       & $31.78$    \\ 
\midrule
UNBIASED               & $93.36$        & $94.04$     & $99.98$        & $75.78$    & $99.85$       & $45.51$ \\ 
\bottomrule
\label{table123}
\end{tabular}
}
\end{table}

\subsection{OOD Generalization}
We compare the OOD generalization performance of our proposed SFP with four state-of-the-art baselines, including three non-structure-based methods: IRM~\citep{arjovsky2019invariant}, REX~\citep{krueger2021out}, DRO~\citep{sagawa2019distributionally}, and one structure-based method: MRM \citep{can}. 
Since MRM and our proposed SFP are both orthogonal to the other three baselines, we also integrate SFP into them (denoted as `SFP+X'), to compare the performance promotion. 
We train the model using two in-domain environments dominated by biased instances and evaluate the performance in an OOD environment. By such means, the upper bound of the OOD generalization performance can be achieved by training the model in an environment with only unbiased instances. 
In our experiment, we define the biased ratio coefficient to indicate the ratio of biased data in two training environments and one testing environment. Specifically, For Full-colored-mnist and Colored-object datasets, we set the biased ratio coefficient as $(0.8,0.6,0.0)$. To increase the difficulty in Scene-object dataset, we set the biased ratio as $(0.9,0.7,0.0)$. The unbiased performance is tested in an environment with $(0.0,0.0,0.0)$. 

The OOD generalization results are demonstrated in Table~\ref{table123}. 
We can notice that the proposed SFP can effectively improve the OOD generalization performance in all cases. The most significant case is the performance in Full-colored-mnist task cooperates with the REX algorithm, which reaches a high accuracy of $93.42\%$. By contrast, the upper bound accuracy of the unbiased case is $94.04\%$, only surpassing our method by $0.62\%$, 
and even though MRM's promotion is adaptable to other baselines, the benefit of MRM is unstable. 
Moreover, in the Full-colored-mnist task, we can notice that MRM and SFP can assist other OOD generalization algorithms, while SFP's promotion is higher. 
However, in more complex tasks such as Colored-object and Scene-object, MRM sometimes has a negative effect. For example, in the Scene-object task, the test accuracy of the DRO algorithm can achieve $31.31\%$ by itself. 
With the integration of MRM, the performance is dragged down to $29.38\%$, while SFP can help to increase the accuracy to $31.78\%$.

\subsection{Loss Tracking}
\setlength{\belowdisplayskip}{2pt}
We visualize the changes in loss values on ERM and the proposed SFP to verify the efficiency of our proposed regularization term in SFP. 
%
It can be seen from Fig.~\ref{fig: loss curve} that the loss of ID instances is always lower than the loss of out-domain instances in the whole training process, which verifies the result of Proposition.~\ref{proposition3}, i.e., SFP can effectively filter out-domain instances by the task loss.

\begin{wrapfigure}{r}{6cm}
\centering
\includegraphics[width=0.4\textwidth]{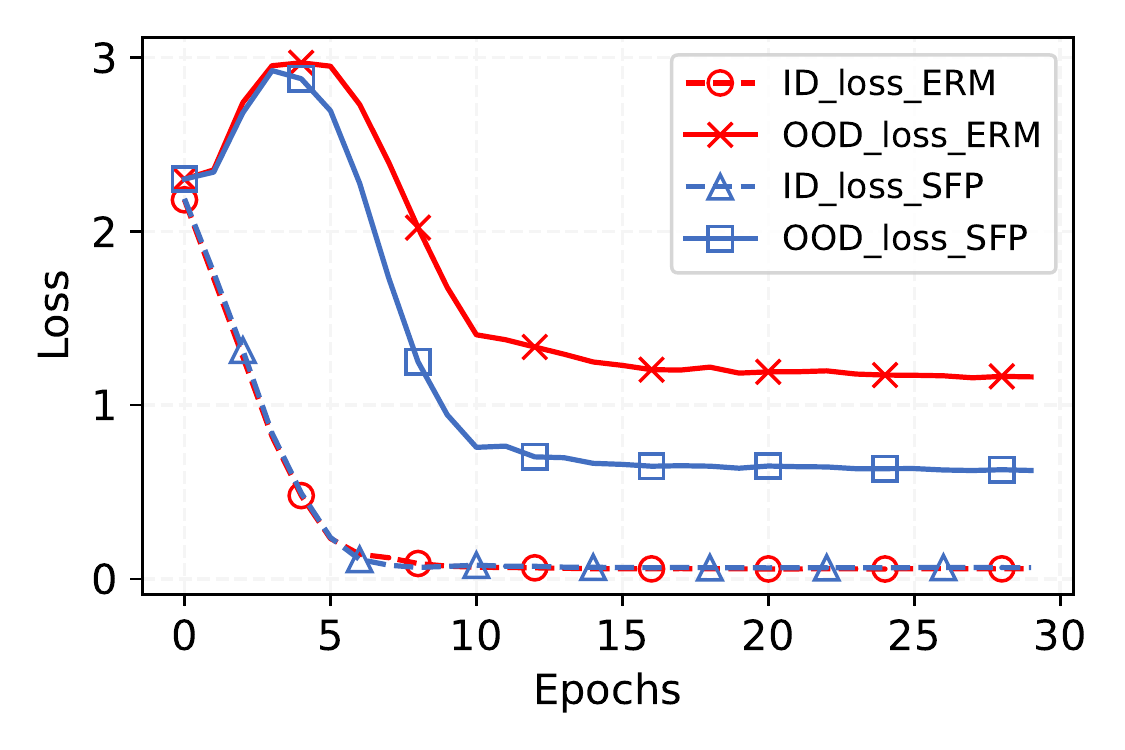}
\caption{Training loss visualization}
\label{fig: loss curve}
\end{wrapfigure}

Besides, We can also notice that the loss of ID instances converges too fast in the ERM algorithm (i.e., the red lines in Fig.~\ref{fig: loss curve}) while the out-domain instances remain large losses, which indicates that the ERM pays too much attention to the biased data, so the model tends to fit the spurious features and ignore the invariant features. 
On the contrary, in the proposed SFP, the distance between ID instances' and out-domain instances' losses is significantly reduced, indicating the effectiveness of spurious feature-targeted pruning. More importantly, the regularization term in our SFP neither slows down the convergence speed nor negatively influences the performance of ID instances.

\section{Conclusion}
In this paper, we propose a novel spurious feature-targeted model pruning framework, dubbed SFP, to automatically explore the optimal model substructure with better OOD generalization. 
By effectively identifying spurious features within ID instances during training, SFP can remove model branches only with strong dependencies on spurious features. Thus, SFP can attenuate the projections of spurious features into the model space and push the model learning toward invariant features. 
We also conduct a detailed theoretical analysis to provide the rationality guarantee and a proof framework for OOD structures via model sparsity. 
Experimental results verified the effectiveness of our method.
\bibliographystyle{plainnat}
\bibliography{name.bib}
\end{document}